\documentclass[12pt]{article}

\usepackage[margin=1in]{geometry}
\usepackage{hyperref}

\usepackage{authblk}

\usepackage[square,numbers]{natbib}
\usepackage{amsmath}
\usepackage{amsthm}
\usepackage{amssymb}

\theoremstyle{plain}
\newtheorem{theorem}{Theorem}
\newtheorem{lemma}{Lemma}
\theoremstyle{definition}

\usepackage{placeins}

\usepackage{graphicx, color}

\usepackage{cleveref}
\usepackage{comment}

\usepackage{mathrsfs}

\begin{document}

\title{Sinusoidal Approximation Theorem for Kolmogorov-Arnold Networks}

%% Alphabetical Order %%
\author[1]{Sergei Gleyzer \thanks{\texttt{sgleyzer@ua.edu}}}
\author[2]{Hanh Nguyen \thanks{\texttt{hvnguyen@ua.edu}}}
\author[1]{Dinesh P. Ramakrishnan \thanks{\texttt{dpr@crimson.ua.edu}}}
\author[1]{Eric A. F. Reinhardt \thanks{\texttt{eareinhardt@crimson.ua.edu}} \thanks{Corresponding author}}
\affil[1]{Department of Physics and Astronomy, University of Alabama}
\affil[2]{Department of Mathematics, University of Alabama}

\maketitle

\begin{abstract}
The Kolmogorov-Arnold representation theorem \cite{kolmogorov1961representation} states that any continuous multivariable function can be exactly represented as a finite superposition of continuous single variable functions. Subsequent simplifications of this representation involve expressing these functions as parameterized sums of a smaller number of unique monotonic functions. These developments led to the proof of the universal approximation capabilities of multilayer perceptron networks with sigmoidal activations, forming the alternative theoretical direction of most modern neural networks \cite{mlp}.

Kolmogorov-Arnold Networks (KANs) have been recently proposed as an alternative to multilayer perceptrons \cite{liu2024kan}. KANs feature learnable nonlinear activations applied directly to input values, modeled as weighted sums of basis spline functions. This approach replaces the linear transformations and sigmoidal post-activations used in traditional perceptrons. Subsequent works have explored alternatives to spline-based activations \cite{rkan, rbfkan, sinekan, wavkan, fractionalkan, chebykan}. In this work, we propose a novel KAN variant by replacing both the inner and outer functions in the Kolmogorov-Arnold representation with weighted sinusoidal functions of learnable frequencies. Inspired by simplifications introduced by Lorentz \cite{lorentz1966approximation} and Sprecher \cite{sprecher1965structure}, we fix the phases of the sinusoidal activations to linearly spaced constant values and provide a proof of its theoretical validity. We also conduct numerical experiments to evaluate its performance on a range of multivariable functions, comparing it with fixed-frequency Fourier transform methods and multilayer perceptrons (MLPs). We show that it outperforms the fixed-frequency Fourier transform and achieves comparable performance to MLPs.

\end{abstract}

\section{Introduction}

The quest to represent complex mathematical functions in terms of simpler constituent parts is a central theme in analysis and applied mathematics. A pivotal moment in this endeavor arose from the list of 23 problems presented by David Hilbert in 1900, which profoundly shaped the mathematical landscape of the 20th century. Hilbert's thirteenth problem, in particular, posed a fundamental question about the limits of function composition. The problem asked whether the solution to the general $7$th degree polynomial equation, considered as a function of its coefficients, a variant of which could be expressed as a finite superposition of continuous functions of only one variable each \cite{Hilbert1927}. 

In the late 1950s, Andreĭ Kolmogorov and his graduate student Vladimir Arnold provided a powerful and affirmative answer to this problem. Their work culminated in the Kolmogorov-Arnold Representation Theorem (KART) \cite{kolmogorov1961representation}, which states that any multivariable continuous function $f$ defined on a compact domain, such as the $n$-dimensional unit cube, can be represented exactly by a finite superposition of single variable functions and the single binary operation of addition. The canonical form of the theorem is expressed as:

\begin{equation}\label{eq:kart}
f(x_{1},\dots,x_{n})=\sum_{q=0}^{2n}\phi_{q}\left(\sum_{p=1}^{n}\psi_{p,q}(x_{p})\right),
\end{equation} 
where $\psi_{p,q}:\mathbb{R}\to\mathbb{R}$ are the \emph{inner} functions and the $\phi_{q}:\mathbb{R}\to\mathbb{R}$ are the \textit{outer} functions which are assumed to be continuous. Lorentz \cite{lorentz1966approximation} and Sprecher \cite{sprecher1965structure} showed that the functions $\phi_{q}$ can be replaced by only one function $\phi$ and that the functions $\psi_{p,q}$ by $\lambda^{pq}\psi_{q}$, where $\lambda$ is a constant and $\psi_{q}$ are monotonic and Lipschitz-continuous functions.

Perceptrons were introduced around the same time by Frank Rosenblatt \cite{rosenblatt1958perceptron} for simple classification tasks. However, their limitations were famously highlighted in the 1969 book "Perceptrons" by Marvin Minsky and Seymour Papert \cite{minsky2017perceptrons}. They proved that a single-layer perceptron could not solve problems that were not linearly separable, such as the simple XOR logical function. This critique led to a significant decline in neural network research for over a decade, often termed the first "AI winter." The revival came in the 1980s with the popularization of the Multilayer Perceptron (MLP), which were introduced in the late 1960s \cite{amari2006theory}. By adding one or more "hidden" layers between the input and output layers, MLPs could overcome the limitations of the single-layer perceptron. The key breakthrough that made these deeper networks practical was the backpropagation algorithm. While developed earlier, its popularization by David Rumelhart, Geoffrey Hinton, and Ronald Williams in 1986 \cite{rumelhart1986learning} provided an efficient method to train the weights in these new, multilayered networks, sparking a renewed wave of interest in the field.

With MLPs demonstrating practical success, the next crucial step was to understand their theoretical capabilities. The Universal Approximation Theorem provides this foundation, showing that a standard MLP with just one hidden layer can, in principle, approximate any continuous function to any desired degree of accuracy. 
George Cybenko  \cite{cybenko1989approximation} used tools from functional analysis like Hahn-Banach theorem and Riesz Representation theorem and showed that given any multivariable continuous function in a compact domain, there exists values of the \emph{weights }and \emph{biases} of a multilayer perceptron with sigmoidal model whose \emph{loss} with respect to the function is bound by any given positive value $\epsilon$.
Almost concurrently, Kurt Hornik, Maxwell Stinchcombe, and Halbert White \cite{hornik1989multilayer} provided a different, more general proof for the universal approximation property of a multilayer perceptron with any bounded, non-constant, and continuous activation function.

The applicability of KART to continuous functions in a compact domain was first highlighted by Hecht-Nielsen \cite{hecht1987kolmogorov,hecht1989neurocomputing}. However, the functions constructed in Kolmogorov's proofs as well as in their later simplifications or improvements are highly complex and non-smooth, which were very different from the much simpler activation functions used in MLPs {\cite{girosi1989representation}. Soon, a series of papers by Věra Kůrková  \cite{Kurkova1991KolmogorovIsRelevant,Kurkova1991KARTandNeuralNetworks} showed that the one-dimensional inner and outer functions could be approximated using 2 layer MLPs and obtained a relationship between the number of hidden units and the loss based on the properties of the function being approximated.

It was not until much later in the 2020s that KART had a direct practical application in the domain of neural networks and machine learning. Kolmogorov-Arnold networks (KANs) were first proposed by Liu et al. \cite{liu2024kan2.0,liu2024kan} as an alternative to MLPs which demonstrated interpretable hidden activations and higher regression accuracy. They used learnable basis splines to model the constituent hidden values similar to KART, and their stacked representation is simplified compared to the original form used in KART. Subsequent works retained the same representation but replaced splines with other series of functions such as radial basis functions and Chebyshev polynomials \cite{li2024fast-kan,qiu2024relu-kan,ss2024chebyshev-kan,xu2024fourierkan} which were computationally faster and demonstrated superior scaling properties with larger numbers of learnable parameters. However, issues of speed and numerical stability on smaller floating point types remain, especially compared to the well-established MLPs \cite{shukla2024comprehensive,yu2024kan}.

To address these challenges, we propose a variant of the representation theorems by Lorentz \cite{lorentz1966approximation} and Sprecher \cite{sprecher1965structure}. We set the phases of the sinusoidal activations to linearly spaced constant values and establish its mathematical foundation to confirm its validity. Previous work has explored this approximation series compared to MLPs and basis-spline approximations and showed competitive performance on the inherently discontinuous domain task of labeling hand-written numerical characters \cite{sinekan}. We extend this work by providing a constructive proof of the approximation series for single variable and multivariable functions and evaluate performance on several such functions with features such as rapid and changing oscillation frequencies. We also extend the comparison from \cite{sinekan} to include fixed-frequency Fourier transform methods and MLP with piecewise-linear and also periodic activation functions. Table \ref{tb:01} summarizes the advancements in this area.

The remainder of this paper is organized as follows. In Section 2, we establish the necessary auxiliary results to prove the approximation theorem for one-dimensional functions. Section 3 extends the results from the previous section to derive a universal approximation theorem for two-layer neural networks. In Section 4, we test several functions and compare the numerical performance of our proposed networks with the classical Fourier transform. Finally, in Section 5, we discuss potential applications and future research directions.

\section{Sinusoidal Universal Approximation Theorem}

The concept of approximating any function using simple, manageable functions has been widely utilized in deep learning and neural networks. In addition to the multilayer perceptron (MLP) approach, the Kolmogorov-Arnold Representation Theorem (KART) has garnered increasing attention from data scientists, who have developed Kolmogorov-Arnold Networks (KANs) to create more interpretable neural networks. In the classical Fourier transform, any continuous, piecewise continuously differentiable ($C^1$) function on a compact interval can be expressed as a Fourier series of sine and cosine functions. In contrast, Kolmogorov demonstrated that any continuous multivariable function can be represented as the form \eqref{eq:kart}.

\begin{table}\label{tb:01}
\centering
\begin{tabular}{|c|c|c|c|}
\hline 
Version & Representation\\
\hline 
\hline 
Kolmogorov (1957) & ${\displaystyle \sum_{q=0}^{2d}g_{q}\left({\displaystyle \sum_{p=1}^{d}\psi_{pq}(x_{p})}\right)}$ \\
\hline 
Lorentz-Sprecher (1965) & ${\displaystyle \sum_{q=0}^{2d}}g\left({\displaystyle \sum_{p=1}^{d}\lambda_{p}\psi(x_{p}+qa)+c_{q}}\right)$ \\
\hline
KAN, B-spline basis (2024) & 
$\begin{aligned}
 y_{i} & ={\displaystyle {\displaystyle {\displaystyle \sum_{j=1}^{d}}\sum_{k=1}^{G+p}w_{ijk}B_{k,p}(x_{j})+b_{i}}}\\
 \implies & {\displaystyle \sum_{i=1}^{H}\sum_{l=1}^{G+p}}W_{il}B_{l}(y_{i})+\beta
\end{aligned}$ \\
\hline 
KAN, sinusoidal basis (ours) & 
$\begin{aligned}
 y_{q} & =\sum_{p=1}^{d}\sum_{k=0}^{N}A_{pq,k}\sin(\omega_{1k}x_{p}+\varphi_{1k})\\
 \implies & \sum_{q=1}^{2d+1}\sum_{j=0}^{N}B_{qj}\sin\left(\omega_{2j}y_q + \varphi_{2j}\right)
\end{aligned}$ \\
\hline
\end{tabular}

\caption{Simplifications proposed to Kolmogorov-Arnold representation theorem at different points of time}

\end{table}

By extending Kolmogorov’s idea and refining both the outer and inner functions using sine terms, we can demonstrate that any continuous function on the interval 
$[0,1]$ can be approximated by a finite sum of sinusoidal functions with varying frequencies and linearly spaced phases.

\begin{theorem}\label{thm:05}
Let $f$ be a continuous function on $[0,1]$ and $0<\alpha\le \frac{\pi}{2}$. For any $\epsilon >0$, there exists $N_0\in\mathbb{N}$ such that for all $N>N_0$, we can find $\omega_k\in [0;2\pi]$ and $A_k\in\mathbb{R}$ for $0\le k\le N$,
\begin{equation}\label{eq:006}
\sup_{0\le x\le 1}\Big|f(x) - \sum_{k=0}^NA_k\sin\left(\omega_kx+\frac{k\alpha}{N+1}\Big)\right| < \epsilon.
\end{equation}
\end{theorem}

To prove this theorem, we need a series of lemmas to approximate the sine function with its Taylor series and then apply the Weierstrass approximation theorem to bring the sum of sinusoidal functions arbitrarily close to any continuous function. The proof is structured to be rigorous and clear, ensuring that readers can follow the logical progression.

\begin{lemma}\label{lem:01}
Let $\omega\in [0, 2\pi]$ and $\alpha\in \mathbb{R}$. For every $N\in\mathbb{N}$,
$$
\sup_{0\le x\le 1}\left|\sin(\omega x+\alpha) - T_N(\omega,\alpha,x) \right|
\le
\dfrac{(2\pi)^{N+1}}{(N+1)!},
$$
where
\begin{equation}\label{eq:001}
T_N(\omega,\alpha,x) = \sum_{l=0}^N\dfrac{\omega^l\sin(\alpha+\frac{l\pi}{2})x^l}{l!}.
\end{equation}
\end{lemma}

\begin{proof}
Let $T_N(\omega,\alpha,x)$, determined by \eqref{eq:001}, be the Taylor polynomial of $f(x)=\sin(\omega x+\alpha)$ centered at $x=0$. Then we have
$$
f(x) = \sin(\omega x+\alpha) = T_N(\omega,\alpha,x) + R_N(x),
$$
where the remainder
$$
R_N(x) = \frac{f^{(N+1)}(\xi)}{(N+1)!}x^{N+1},
$$
for some $\xi\in (0,1)$.

Notice that $\big|f^{(N+1)}(\xi) \big|\le \omega^{N+1}\le (2\pi)^{N+1}$ for all $\xi\in (0,1)$. This completes the proof of Lemma \ref{lem:01}.

\end{proof}

\begin{lemma}\label{lem:02}
For every $\epsilon >0$, there exists $N_0\in\mathbb{N}$ such that for all $N>N_0$,
\begin{equation}\label{eq:002}
\sup_{0\le x\le 1}\left| \sum_{k=0}^NA_k\sin(\omega_kx+\alpha_k) - \sum_{k=0}^NA_kT_N(\omega_k,\alpha_k,x) \right|
< \epsilon
\end{equation}
for all $\omega_k\in [0,2\pi]$, $\alpha_k,A_k\in\mathbb{R}$, $0\le k\le N$, where $T_N(\omega_k,\alpha_k,x)$ is determined by \eqref{eq:001}.
\end{lemma}

\begin{proof}
By Lemma \ref{lem:01}, we have
$$
\sup_{0\le x\le 1}\left|\sin(\omega_k x+\alpha_k) - T_N(\omega_k,\alpha_k,x) \right|
\le
\dfrac{(2\pi)^{N+1}}{(N+1)!}
$$
for all $\omega_k\in [0,2\pi]$, $\alpha_k\in\mathbb{R}$, $0\le k\le N$, and $N\in\mathbb{N}$.
Therefore
\begin{align*}
\sup_{0\le x\le 1}\left| \sum_{k=0}^NA_k\sin(\omega_kx+\alpha_k) - \sum_{k=0}^NA_kT_N(\omega_k,\alpha_k,x) \right|
\le&
\sum_{k=0}^N|A_k|\sup_{0\le x\le 1}\Big| \sin(\omega_kx+\alpha_k) -  T_N(\omega_k,\alpha_k,x) \Big|\\
\le &
\sum_{k=0}^N\dfrac{(2\pi)^{N+1}|A_k|}{(N+1)!} \le \dfrac{(2\pi)^{N+1}}{N!}\max_{k}|A_k|.
\end{align*}

For given $\epsilon>0$, we can choose $N_0\in\mathbb{N}$ such that for all $N>N_0$, we obtain  \eqref{eq:002} to complete the proof of Lemma \ref{lem:03}.

\end{proof}

\begin{lemma}[Weierstrass Approximation, \cite{sapm192541148}]\label{lem:03}
Let $f$ be a continuous function on $[0,1]$. Then for any $\epsilon >0$, there exists $N_0$ such that for all $N>N_0$
$$
\sup_{0\le x\le 1}\left| f(x) - B_N(f,x) \right| < \epsilon,
$$
where $B_N(f,x)$ is the Bernstein polynomial of $f$ determined by
\begin{equation}\label{eq:003}
B_N(f,x) = \sum_{l=0}^N f\Big(\frac{l}{N}\Big)\binom{N}{l}x^l(1-x)^{N-l}.
\end{equation}

\end{lemma}

\begin{lemma}\label{lem:04}
Let $p(x) = \sum\limits_{l=0}^Nb_lx^l$ be any polynomial and $0<\alpha_k<\frac{\pi}{2}$ for all $0\le k\le N$. Then there exist $\omega_k\in [0,2\pi]$ and $A_k\in \mathbb{R}$ for $0\le k\le N$ such that
\begin{equation}\label{eq:004}
p(x) = \sum_{k=0}^NA_kT_N(\omega_k,\alpha_k,x)
\end{equation}
for all $x\in [0,1]$, here $T_N(\omega_k,\alpha_k,x)$ is determined by \eqref{eq:001}.
\end{lemma}

\begin{proof}
Recall the formula determined by \eqref{eq:001} then we have
\begin{align*}
\sum_{k=0}^NA_kT_N(\omega_k,\alpha_k,x) = & \sum_{k=0}^N\sum_{l=0}^NA_k\dfrac{\omega_k^l\sin(\alpha_k+\frac{l\pi}{2})x^l}{l!}\\
 =& \sum_{l=0}^N \left( \sum_{k=0}^N \dfrac{A_k\omega_k^l\sin(\alpha_k+\frac{l\pi}{2})}{l!} \right)x^l.
\end{align*}
To obtain \eqref{eq:004}, we need to choose $\omega_k$ and $A_k$ such that
$$
\sum_{k=0}^N \dfrac{A_k\omega_k^l\sin(\alpha_k+\frac{l\pi}{2})}{l!} = b_l,\qquad 0\le l\le N.
$$
Equivalently, we need to find $\omega_k$ and $A_k$ such that
\begin{equation}\label{eq:005}
\sum_{k=0}^NA_k\omega_k^l\sin(\alpha_k+\frac{l\pi}{2}) = b_l\cdot l!,\qquad 0\le l\le N.
\end{equation}

Consider the following $(N+1)\times (N+1)$ matrix
\[
M = \begin{bmatrix}
\sin(\alpha_0) & \sin(\alpha_1)&\cdots&\sin(\alpha_N)\\
\omega_0\sin(\alpha_0+\frac{\pi}{2})&\omega_1\sin(\alpha_1+\frac{\pi}{2})&\cdots&\omega_N\sin(\alpha_N+\frac{\pi}{2})\\
\vdots&\ddots&\cdots&\vdots\\
\omega_0^N\sin(\alpha_0+\frac{N\pi}{2})&\omega_1^N\sin(\alpha_1+\frac{N\pi}{2})&\cdots&\omega_N^N\sin(\alpha_N+\frac{N\pi}{2})
\end{bmatrix}.
\]
Since $0<\alpha_k<\frac{\pi}{2}$, $\sin(\alpha_k+\frac{l\pi}{2})\ne 0$ for all $0\le l\le N$. By induction in $N$, we can select $\omega_0,\omega_1,\ldots,\omega_N\in [0;2\pi]$ such that $\det(M)\ne 0$. Therefore the system of equations \eqref{eq:005} with the augmented matrix $M$ has a solution for $A_0,A_1,\ldots,A_N$. This completes the proof of Lemma \ref{lem:04}.
\end{proof}

\begin{proof}[Proof of Theorem \ref{thm:05}]
For given $\epsilon >0$, by Lemmas \ref{lem:02} and \ref{lem:03}, there exist $N_0\in\mathbb{N}$ such that for all $N>N_0$ we have
\begin{equation}\label{eq:007}
\sup_{0\le x\le 1}\left| f(x) - p(x) \right| < \epsilon/2,
\end{equation}
where $p(x) = B_N(f,x)$ is the Bernstein polynomial of $f$,
and
\begin{equation}\label{eq:008}
\sup_{0\le x\le 1}\left| \sum_{k=0}^NA_k\sin(\omega_kx+\alpha_k) - \sum_{k=0}^NA_kT_N(\omega_k,\alpha_k,x) \right|
< \epsilon/2
\end{equation}
for some $\omega_k,A_k$  to be chosen later and and $\alpha_k = \frac{k\alpha}{N+1}\in (0,\frac{\pi}{2})$.

By virtue of Lemma \ref{lem:04}, we can find $\omega_k\in[0,2\pi]$ and $A_k\in\mathbb{R}$ for $0\le k\le N$ such that
$$
p(x) = \sum_{k=0}^NA_kT_N(\omega_k,\alpha_k,x).
$$
Now combining \eqref{eq:007} and \eqref{eq:008}, we obtain the conclusion of Theorem \ref{thm:05}.

\end{proof}

\section{Sinusoidal Approximation Theorem for Two-layer Neural Networks}

We next extend the result of Theorem \ref{thm:05} to approximate any continuous function defined on a compact domain in $\mathbb{R}^n$. To simplify the treatment of function extension, we restrict our attention to functions supported within the unit cube $I^n = [0,1]^n$. Under this setting, a two-layer neural network with sinusoidal activation functions can approximate any continuous function on the unit cube. A similar result for sigmoidal activation functions was established by Kůrková in \cite{Kurkova1991KARTandNeuralNetworks}. The main goal of this section is to approximate a continuous function on a compact domain by a finite sum of nested sinusoidal functions of the form:
\begin{equation}\label{eq:10}
\sum_{j=0}^MB_{j}\sin\Big(\nu_{j}
\sum_{p=1}^n
\sum_{k=0}^NA_{pk}^j\sin(\omega_{k}x_p+\varphi_{k})+\gamma_{j}).
\end{equation}

\begin{theorem}\label{thm:06}
Let $f: I^n\to \mathbb{R}$ be a function. The for any $\varepsilon >0$, there exist $A_{pq,k},\omega_{1k},\varphi_{1k},\quad 0\le k\le N$ and $B_{qj},\omega_{2j},\varphi_{2j},\quad 0\le j\le M$ such that
\begin{equation}\label{eq:006}
\Big|
f(x) - \sum_{q=1}^{2n+1}\sum_{j=0}^MB_{qj}\sin\Big(\omega_{2j}
\sum_{p=1}^n
\sum_{k=0}^NA_{pq,k}\sin(\omega_{1k}x_p+\varphi_{1k})+\varphi_{2j})
\Big|
<\varepsilon.
\end{equation}
\end{theorem}

\begin{proof}

By Kolmogorov Theorem, any continuous function $f: I^n\to \mathbb{R}$ can be written as
$$
f(x_1,\ldots,x_n) = \sum_{q=1}^{2n+1}\phi_q(\sum_{p=1}^n\psi_{pq}(x_p)),
$$
where $\phi_q$ and $\psi_{pq}$ are continuous functions on the real number line.

Each function $\psi_{pq}$ is continuous, the sum $\sum_{p=1}^n\psi_{pq}(x_p)$ is bounded for all $x\in I^n$. It means that
$$
a\le \sum_{p=1}^n\psi_{pq}(x_p)\le b
$$
for some $a,b\in\mathbb{R}$ and for all $x\in I^n$.

For each function $\psi_{pq}(x_p)$, we need to approximate it by Theorem \ref{thm:05}
$$
\mid \psi_{pq}(x_p) - \sum_{k=0}^NA_{pq,k}\sin(w_{1k}x_p+\varphi_{1k}) \mid < \delta_1/n.
$$
Therefore
$$
\Big|
\sum_{p=1}^n\psi_{pq}(x_p)
-\sum_{p=1}^n \Big(\sum_{k=0}^NA_{pq,k}\sin(w_{1k}x_p+\varphi_{1k})\Big) 
\Big|
< \delta_1.
$$

Now each function $\phi_q$ is uniformly continuous on $[a, b]$, so
$
|\phi_q(u) - \phi_q(v)| <\varepsilon/(4n+2)
$
whenever $|u-v|<\delta_1$.
Now we have

\begin{equation}\label{eq:psi1}
\mid\phi_q(\sum_{p=1}^n\psi_{pq}(x_p)) - \phi_q\Big(\sum_{p=1}^n \Big(\sum_{k=0}^NA_{pq,k}\sin(w_{1k}x_p+\varphi_{1k})\Big) \Big)\mid < \varepsilon/(4n+2).
\end{equation}

Next, we apply Theorem \ref{thm:05} to $\phi_q$ again to obtain
\begin{equation}\label{eq:psi2}
\mid
\phi_q(\sum_{p=1}^n\sum_{k=0}^NA_{pq,k}\sin(w_{1k}x_p+\varphi_{1k}) ) - 
\sum_{j=0}^MB_{qj}\sin\Big(w_{2j}\Big(\sum_{p=1}^n \Big(\sum_{k=0}^NA_{pq,k}\sin(w_{1k}x_p+\varphi_{1k})  \Big) \Big)+\varphi_{2j})  
\mid < \delta_2
\end{equation}
where $\delta_2 = \varepsilon/(4n+2)$.

We now combine equations \eqref{eq:psi1} and \eqref{eq:psi2} to complete the proof of Theorem \ref{thm:06}.

\end{proof}

\section{Numerical Analysis}

To evaluate the performance of our sinusoidal approximation, we compare it with the Fourier series for approximating one-dimensional functions. For higher-dimensional problems with two inputs and one output, we benchmark our approach against two-layer (MLPs) \cite{mlp}, using either ReLU \cite{relu} or sine activation functions, as well as against multi-dimensional truncated Fourier series. In both cases, we consider functions with a single output and compute the relative $L^2$ error as follows:

\begin{equation}
    \text{Relative $L^2$ Error} \;=\; \frac{\left\lVert \mathbf{y}_{\text{values}} - \mathbf{y}_{\text{fit}} \right\rVert_2}{\left\lVert \mathbf{y}_{\text{values}} \right\rVert_2}.
\end{equation}

Based on Theorem \ref{thm:05}, we construct our SineKAN model and analyze its performance numerically using the following neural network formulation for one-dimensional functions:

\begin{equation}
    y = \sum_k^G A_k\sin(\omega_k x + k/(G+1))+b,
\end{equation}
where $\omega_k$, $A_k$, and $b$ are learnable frequency parameters, amplitude functions, and a bias term, respectively. For multi-dimensional functions, Theorem \ref{thm:06} guides the construction of SineKAN layers as follows:

\begin{gather}
    y_j = \sum_k^G\sum_l^N A_{jkl}\sin(\omega_k x_l + \frac{k}{G+1} + \frac{l\pi}{N+1} )+b_j,\\
    z_m = \sum_n^G\sum_j^N B_{mnj}\sin(\nu_n y_j + \frac{j}{G+1} + \frac{n\pi}{N+1})+c_m,
\end{gather}
where $A_{jkl}$ and $B_{mnj}$ are learnable amplitude tensors, $b_j$ and $c_m$ are learnable bias vectors, and $\omega_k$ and $\nu_n$ are learnable frequency vectors. For the one-dimensional case, we consider functions defined on a uniform grid of input values from $0.01$ to $1$. These functions pose challenges for convergence in Fourier series due to their singularities or non-periodicity:

The first function,
\begin{equation}\label{eq:1Dfunc1}
    f(x) = e^{-\frac{1}{x}} \sin\left(\frac{1}{x}\right),
\end{equation}
is non-periodic, has small magnitude across the domain, and exhibits a strong singularity at $x = 0$. 

The second function,
\begin{equation}\label{eq:1Dfunc2}
    f(x) = \sum_{k} e^{\frac{kx}{\pi}} \sin(kx),
\end{equation}
shows rapid growth and high-frequency oscillations near $x = 0$.

The third function,
\begin{equation}\label{eq:1Dfunc3}
    f(x) = \sum_{k} e^{-\frac{1}{x}} \sin\left(kx + \frac{\pi}{k}\right),
\end{equation}
incorporates phase shifts to evaluate the model's performance and convergence with respect to linearly spaced phases.

The final two functions are particularly challenging for Fourier series convergence, allowing us to test our model's convergence behavior:
\begin{equation}\label{eq:1Dfunc4}
    f(x) = x^{\frac{1}{5}} \sin\left(\frac{1}{x}\right),
\end{equation}
\begin{equation}\label{eq:1Dfunc5}
    f(x) = x^{\frac{4}{5}} \sin\left(\frac{1}{x}\right).
\end{equation}

All models are fitted using the Trust Region Reflective algorithm for least-squares regression from the \texttt{scipy} package \cite{trf}. Each function is fitted for a default of 100 steps per fitted parameter.

\begin{figure}[ht!]
    \centering
    \includegraphics[width=1.\linewidth]{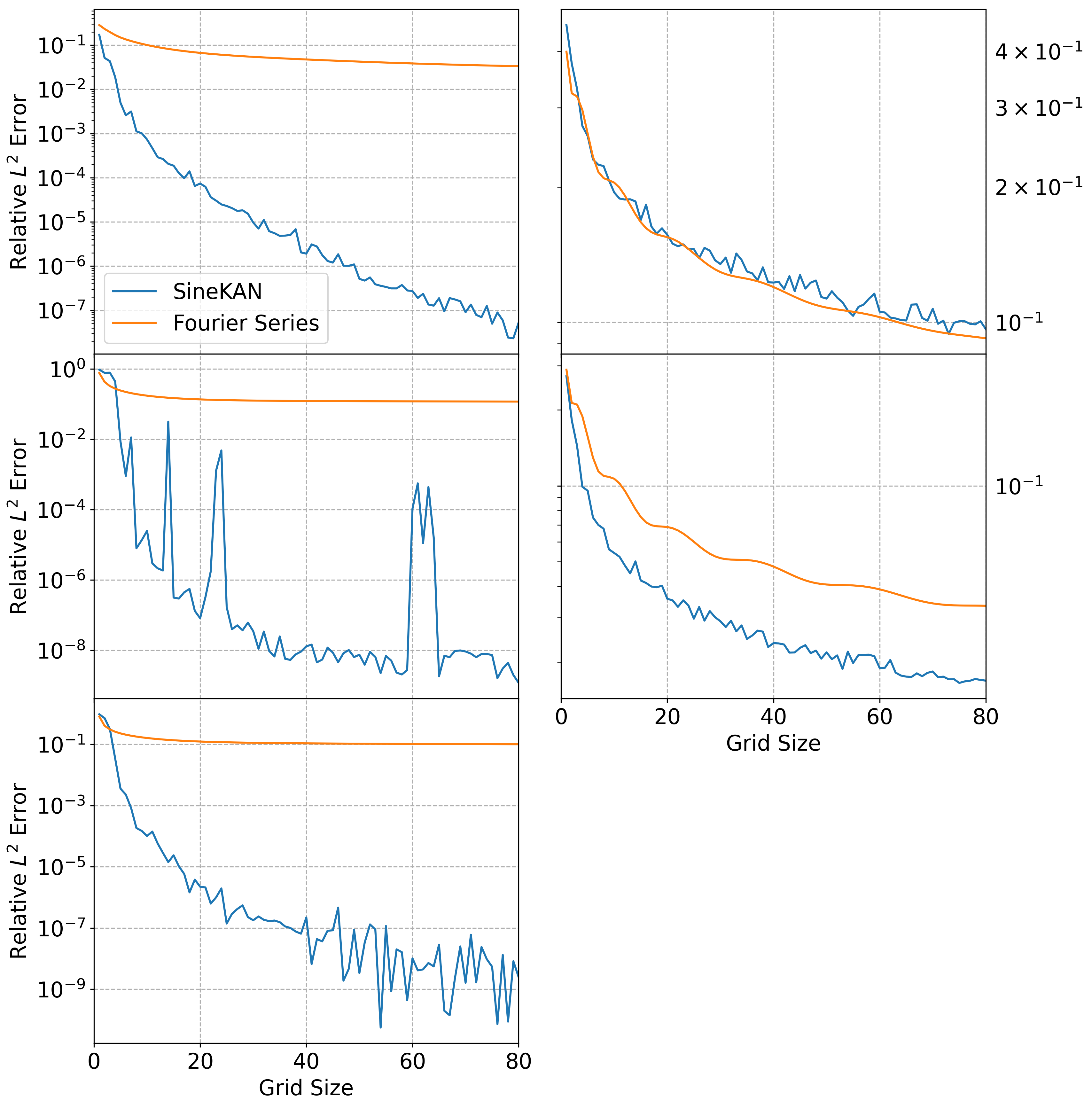}
    \caption{Approximation error as a function of grid size: top left \cref{eq:1Dfunc1}, middle left \cref{eq:1Dfunc2}, bottom left \cref{eq:1Dfunc3}, top right \cref{eq:1Dfunc4}, middle right \cref{eq:1Dfunc5}.}
    \label{fig:1Dapprox}
\end{figure}

For \eqref{eq:1Dfunc1}, \eqref{eq:1Dfunc2}, and \eqref{eq:1Dfunc3}, the SineKAN approximation significantly outperforms the Fourier series approximation. In \eqref{eq:1Dfunc4}, performance is roughly comparable between the two. We observe that the function in \eqref{eq:1Dfunc4} has less regularity, which causes both the SineKAN and the Fourier series to converge slowly.

For multidimensional functions, we benchmark the following two equations on a 100 by 100 mesh grid of input values ranging from 0.01 to 1:

The first function,
\begin{equation}\label{eq:2Dfunc1}
    f(x, y) = x^2 + y^2 - a e^{-\frac{(x-1)^2 + y^2}{c}} - b e^{-\frac{(x+1)^2 + y^2}{d}},
\end{equation}
with parameters $a = \frac{3}{2}$, $b = 1$, $c = 0.5$, and $d = 0.5$, features Gaussian-like terms that create a complex surface, suitable for testing convergence on smooth but non-trivial landscapes.

The second function (Rosenbrock function),
\begin{equation}\label{eq:2Dfunc2}
    f(x, y) = (a - x)^2 + b (y - x^2)^2,
\end{equation}
with parameters $a = 1$ and $b = 2$, represents a non-linear, non-symmetric surface, ideal for evaluating convergence in challenging multidimensional optimization problems.

\begin{figure}[ht!]
    \centering
    \includegraphics[width=1.\linewidth]{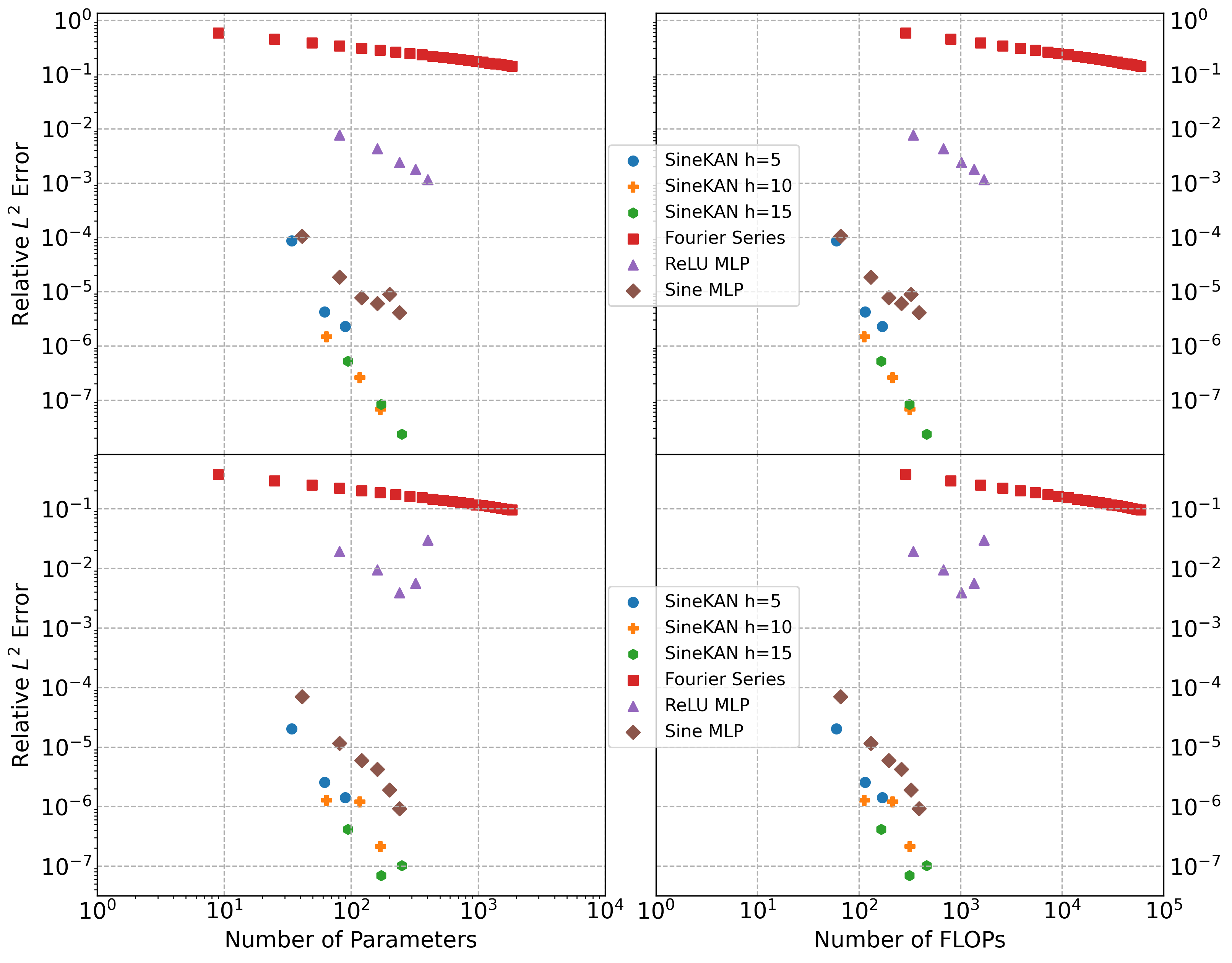}
    \caption{Loss as a function of number of parameters (Left) and FLOPs (Right) for \cref{eq:2Dfunc1} (Top) and \cref{eq:2Dfunc2} (Bottom).}
    \label{fig:2Dparams}
\end{figure}

%% We could try to test the Rastrigin function: https://en.wikipedia.org/wiki/Rastrigin_function

We show in Figure \ref{fig:2Dparams} that the two-layer SineKAN outperforms the two-layer MLP with sinusoidal activation functions as a function of the number of parameters. MLP with ReLU activations perform substantially worse with several orders of magnitude higher error as a function of the number of parameters and Fourier series with a characteristic error roughly one to two orders of magnitude greater than the two-layer MLP with ReLU.

We also compute performance as a function of the number of relative FLOPs or compute units. To do this calculation, we run 10 million iterations using numpy arrays of size 1024 to estimate the relative compute time of addition, multiplication, ReLU, and sine and find that, when setting addition and multiplication to approximately 1 FLOP, ReLU costs an estimated 1.5 FLOPs, and sine functions cost 12 FLOPs. We carry out similar estimates in pytorch and find that relative FLOPs for sine would be closer to 3.5 in pytorch and relative FLOPs for ReLU would be around 1 FLOP. The \cref{fig:2Dparams} is based on numpy estimates.

\FloatBarrier

\section{Discussion}
The original implementation of the KAN model developed by Liu et al. used basis-spline functions \cite{liu2024kan}. These were proposed as an alternative to MLP due to their improved explainability, domain segmentation, and strong numerical performance in modeling functions. However, later work showed that, when accounting for increases in time and space complexity, the basis-spline KAN underperformed MLP \cite{fairkancompare}. Previous work on the Sirens model has shown that, for function modeling, particularly for continuously differentiable functions, sinusoidal activations can improve the performance of MLP architectures \cite{siren}. This motivated the development of the SineKAN architecture, which builds on both concepts by combining KAN's learnable on-edge activation functions and on-node weighted summation with the periodic activations from Sirens \cite{sinekan}.

We further extend this work by providing a robust constructive proof for the approximation power of the SineKAN model. We show that a single layer is sufficient for approximation of arbitrary 1D functions and that a two-layer SineKAN is sufficient for approximation of arbitrary multivariable functions bounded by the same constraints of the original KART \cite{kolmogorov1957representations}.

We show in \cref{fig:1Dapprox} and \cref{fig:2Dparams} that these functions can achieve low errors in modeling mathematical functions with features such as rapid- and variable-frequency oscillations. For 2D functions, we show that SineKAN outperforms MLP, including MLP with sinusoidal activations, with flexible model parameter combinations when accounting for both time and space complexity of the models. This strongly motivates further exploration of this model for numerical approximation tasks.

Given the inherent periodic nature of sinusoidal functions, our approximation framework \eqref{eq:006} shows strong potential for modeling periodic and time-series data. Future work will explore the extension of SineKAN to continual learning tasks, particularly in scenarios involving dynamic environments or non-stationary data. Further directions include theoretical analysis of generalization bounds, integration with neural differential equations, and applications in signal processing and real-time prediction systems.

\section{Conclusion}
In this paper, we build upon Lorentz's and Sprecher's foundational work \cite{lorentz1966approximation,sprecher1965structure} to establish two main theorems for approximating single and multi-variable functions using the sinusoidal activation function. Our proposed SineKAN models introduce learnable frequencies, amplitudes, and biases, offering a flexible and expressive framework for function approximation. Through numerical experiments, we demonstrate that SineKAN outperforms the classical Fourier series and MLP in accuracy across a variety of test cases. To support reproducibility and ease of experimentation, we provide a link to our code below.

\section{Code Availability}
The source code for our implementation is available in a GitHub repository: 
\url{https://github.com/ereinha/SinusoidalApproximationTheorem}.

\section{Acknowledgement}
This work was supported by the U.S. Department of Energy (DOE) under Award No. DE-SC0012447 (D.R., E.R., and S.G.).

%\printbibliography
\bibliography{references}

\end{document}